\documentclass{article}
\pdfoutput=1 
\usepackage{kr}

\usepackage{times}
\usepackage{soul}
\usepackage{url}
\usepackage[hidelinks]{hyperref}
\usepackage[utf8]{inputenc}
\usepackage[small]{caption}
\usepackage{graphicx}
\usepackage{amsmath}
\usepackage{amsthm}
\usepackage{booktabs}
\usepackage{algorithm}
\usepackage[noend]{algpseudocode}
\usepackage{macros}

\newtheorem{example}{Example}
\newtheorem{theorem}{Theorem}
\newtheorem{lemma}[theorem]{Lemma}

\title{On the Approximability of Weighted Model Integration on DNF Structures}

\author{ Ralph Abboud \and
	{\.I}smail {\.I}lkan Ceylan\And
	Radoslav Dimitrov
	\affiliations
	Department of Computer Science, University of Oxford\\
	\emails
	\{firstname.lastname\}@cs.ox.ac.uk
}

\begin{document}

\maketitle

\begin{abstract}

\emph{Weighted model counting} (\WMC) consists of computing the weighted sum of all satisfying assignments of a propositional formula. \WMC is well-known to be \sharpP-hard for exact solving, but admits a fully polynomial randomized approximation scheme (FPRAS) when restricted to DNF structures. In this work, we study \emph{weighted model integration}, a generalization of weighted model counting which involves real variables in addition to propositional variables, and pose the following question: Does weighted model integration on \DNF structures admit an FPRAS?  
Building on classical results from approximate volume computation and approximate weighted model counting, we show that weighted model integration on \DNF structures can indeed be approximated for a class of weight functions. Our approximation algorithm is based on three subroutines, each of which can be a \emph{weak} (i.e., approximate), or a \emph{strong} (i.e., exact) oracle, and in all cases, comes along with accuracy guarantees. We experimentally verify our approach over randomly generated \DNF instances of varying sizes, and show that our algorithm scales to large problem instances, involving up to 1K variables, which are currently out of reach for existing, general-purpose weighted model integration solvers.	
\end{abstract}

\section{Introduction}
\emph{Weighted model counting}~(\WMC) has been introduced as a unifying approach for encoding probabilistic inference problems that arise in various formalisms.
Informally, given a propositional formula, and a weight function that assigns every truth assignment a weight, \WMC amounts to computing the weighted sum of all the satisfying assignments \cite{Gomes09}.
Many probabilistic inference problems in \emph{probabilistic graphical models}~\cite{Koller-PGM}, \emph{probabilistic planning}~\cite{Domshlak07}, \emph{probabilistic logic programming}~\cite{ProbLog}, \emph{probabilistic databases}~\cite{Suciu-PDBs}, and \emph{probabilistic knowledge bases}~\cite{BCL-AAAI17} can be reduced to a form of \WMC. 

Despite its wide applicability, \WMC is limited to discrete domains and thus cannot be applied to domains involving real variables,  and this motivated the study of \emph{weighted model integration}~(\WMI)~\cite{Belle15}, as a generalization of \WMC. 

Building on the foundations of \emph{satisfiability modulo theories~(SMT)}~\cite{BSST09}, \WMI can capture \emph{hybrid domains} with mixtures of Boolean and continuous variables. 
Briefly, the input to \WMI is a \emph{hybrid} propositional formula that additionally involves arithmetic constraints (e.g., linear constraints over real, or integer variables), and a weight function that defines a \emph{density} for every truth assignment of the formula. \WMI is then the task of computing the sum of \emph{integrals} over the densities of all the satisfying assignments of the given hybrid propositional formula \cite{Belle15,MorettinAIJ19}.

The standard formulation of \WMI assumes a formula in \emph{conjunctive normal form}~(\CNF) as an input, and to date, there is no study of \WMI which is specifically tailored to formulas in \emph{disjunctive normal form}~(\DNF). This is surprising, as both variants are widely investigated for \WMC. We write  $\WMI(\CNF)$ and $\WMI(\DNF)$ in the sequel to distinguish between these cases. These problems are clearly $\sharpP$-hard for \emph{exact solving}, as are their respective special cases $\WMC(\CNF)$ and $\WMC(\DNF)$ \cite{Valiant79}.
For \emph{approximate solving}, however, there is a strong contrast in computational complexity between variants of weighted model counting problems: $\WMC(\DNF)$ has a \emph{fully polynomial randomized algorithm scheme (FPRAS)}~\cite{Karp89}, producing polynomial-time approximations with guarantees, whereas $\WMC(\CNF)$ is $\NP$-hard to approximate \cite{Roth96}. 
The latter polynomial-time inapproximability result immediately propagates to $\WMI(\CNF)$, while the approximability status of $\WMI(\DNF)$ remains \emph{open}. In this paper, we pose the following question: Does $\WMI(\DNF)$ admit an FPRAS?

We answer this question in the affirmative, and provide a polynomial-time algorithm for $\WMI(\DNF)$ with probabilistic accuracy guarantees. 
The intuition behind our result is based on two observations. First, the special case of $\WMI(\DNF)$ without any arithmetic constraints corresponds to $\WMC(\DNF)$ which has an FPRAS~\cite{Karp89}. Second, the special case of $\WMI(\DNF)$ with constant weight functions, and without any Booleans, corresponds to computing the volume of \emph{unions of convex bodies}, which also has an FPRAS \cite{Bringmann}.
Our result builds on these results, and extends them, by allowing extra constructs essential for \WMI, while preserving the approximation guarantees. Our main contributions can be summarized as follows:
\begin{enumerate}[--]
	\item We propose an efficient approximation algorithm for $\WMI(\DNF)$, called \approxwmi, extending  the algorithm given in \cite{Bringmann}. 

	\item We prove that \approxwmi is an FPRAS provided that the weight functions are \emph{concave}, and can be factorized into products of weights of literals. We provide asymptotic bounds for the running time of the algorithm. 
	\item We extend \approxwmi to the case where the products of weights assumption is relaxed, and provide asymptotic bounds for the running time of the algorithm. 
	\item We experimentally verify our approach, using a strong oracle for computing the volume of a body.  Our experiments suggest that \approxwmi solves large problem instances, including up to 1K variables, which are out of reach for any existing, general-purpose \WMI solver.
\end{enumerate}
The full proofs of our results can be found in the appendix of this paper.

\section{Preliminaries}
We briefly recall \emph{propositional logic}, \emph{linear real arithmetic} and \emph{weighted model integration}, where we also settle the notation and assumptions used throughout the paper.

\subsection{Logic and Linear Real Arithmetic}
Let us denote by $\mathbb{R}$ the real domain, and by $\mathbb{B}$ the Boolean domain $\{0,1\}$.
Let $\Xbf$ be a set of $n$ real variables, and $\Vbf$ be a set of $m$ Boolean variables (or atoms). An \emph{LRA atom} is of the form 
\begin{align*}
\sum_i c_i x_i \bowtie c,
\end{align*} where $c$ and $c_i$ are rational values/constants, ${x_i\in \Xbf}$, and ${\bowtie \in \{<, \leq, >,\geq, =, \neq\}}$ with their usual semantics. We write $\atoms(\Xbf,\Vbf)$ to denote the set of atoms over $\Xbf \cup \Vbf$.  A \emph{literal} is either an atom or its negation. We sometimes write \emph{LRA literal}, or a \emph{Boolean literal}, to distinguish the literals depending on the domain of their corresponding variables.

A \emph{propositional formula $\phi$} over $\atoms(\Xbf,\Vbf)$ is defined as a Boolean combination of literals via the logical connectives $\{\neg, \land, \lor, \to, \leftrightarrow\}$.  If $\Vbf=\emptyset$, we say $\phi$ is an \emph{LRA formula}, and if $\Xbf=\emptyset$ then $\phi$ corresponds to a standard propositional formula defined only over Boolean variables. 

\begin{example} Let us consider the formula $\phi_{ex}$ given as
	\[
	 ((0 \leq x_1 \leq 5) \lor \neg p_1) \land (p_2 \lor \neg (10 \leq x_1 + x_2 \leq 15)),
	\]
which contains 2 Boolean and 2 LRA literals. \hfill \qedboxfull
\end{example}	

Given a propositional formula $\phi$ over $\atoms(\Xbf,\Vbf)$, a \emph{truth assignment} $\tau: \atoms(\Xbf,\Vbf) \mapsto \mathbb{B}$, maps every atom to either $0$ (false), or $1$ (true). A truth assignment $\tau$ \emph{satisfies} a propositional formula $\phi$, denoted $\tau\models \phi$, in the usual sense, where $\models$ is the propositional entailment relation. We  sometimes say $\tau$ \emph{propositionally satisfies} $\phi$ to make the underlying entailment relation explicit. 

Observe that a propositional formula  over $\atoms(\Xbf,\Vbf)$ may have a propositionally satisfying truth assignment but not admit a solution to the LRA constraints, i.e., the relevant LRA constraints define an empty polytope. An assignment $\tau$ is \emph{LRA-satisfiable} if the solution space to the set of linear inequalities induced by the mapping $\tau$ is non-empty. The classical SMT problem over LRA constraints is, for a given a propositional formula $\phi$ over $\atoms(\Xbf,\Vbf)$, to decide whether there exists an assignment $\tau$ such that $\tau\models \phi$  (propositionally satisfiable), and $\tau$ is LRA-satisfiable.

\begin{example}
	Consider again the formula $\phi_{ex}$, and an assignment $\tau$ with ${\tau(p_1)=1}$, ${\tau(p_2)=0}$, ${\tau(0 \leq x_1 \leq 5)=1}$, and ${\tau(10 \leq x_1 + x_2 \leq 15)=0}$. Clearly, $\phi_{ex}$ is propositionally satisfiable as witnessed by $\tau$; it is also LRA satisfiable, e.g., for values $x_1=2$, and $x_2=30$, the assignment $\tau$ is LRA satisfiable.
	By contrast, the formula 
	\[
	\phi_{ex'} = (1 \leq x_1 + x_2 \leq 4) \land \neg p_1 \land (x_1 \leq -2) \land (x_2 \leq 2)
	\] 
	is trivially propositionally satisfiable, but is not LRA satisfiable, since $(x_1 \leq -2) \land (x_2 \leq 2)$ and $(1 \leq x_1 + x_2 \leq 4)$ cannot be satisfied simultaneously.
\hfill \qedboxfull
\end{example}	

We recall the fragments of propositional logic. A \emph{conjunctive clause} is a conjunction of literals, and a \emph{disjunctive clause} is a disjunction of literals. A propositional formula $\phi$ is in \emph{conjunctive normal form} (\CNF) if it is a conjunction of disjunctive clauses, and it is in \emph{disjunctive normal form} (\DNF) if it is a disjunction of conjunctive clauses. A clause has \emph{width} $k$ if it has exactly $k$ literals. We say that a \DNF (resp., \CNF) has width $k$ if it contains clauses of width at most $k$. 

\subsection{Weighted Model Integration}
\label{ssec:wmi_def}

Let $\Xbf$ be a set of $n$ real variables, and $\Vbf$ a set of $m$ Boolean variables. We consider a weight function ${w: (\Rbb^n \times \mathbb{B}^m) \mapsto \Rbb^+}$,  and propositional formulas $\phi(\Xbf,\Vbf)$ such that ${\phi: (\Rbb^n \times \mathbb{B}^m) \mapsto \mathbb{B}}$. 

The \emph{weighted model integral}  is defined as:
\begin{align}
\label{wmi}
\WMI(\phi, w \mid \Xbf, \Vbf) = \sum_{\vbf} \int\limits_{\xbf_{\phi}} w(\xbf,\vbf) d \xbf 
\end{align}
where $\vbf$ is an assignment to Boolean variables $\Vbf$, and $\xbf_{\phi}$ denotes the real valuations of $\Xbf$ satisfying $\phi(\xbf,\vbf)$. 

Note that the weight function depends both on Boolean and real variables. It is common to employ a simplifying assumption to $w(\xbf, \vbf)$ (e.g.,~Section \ref{ssec:wmi_def} of \cite{Martires19}): 
\begin{align}
\label{fact}
w(\xbf, \vbf) = w_x(\xbf) \prod_{i=1}^m w_b(p_i),
\end{align}
where $w_x: \mathbb{R}^n \rightarrow \mathbb{R}^+$ and $w_b: \mathbb{B} \rightarrow ]0,1[$ are functions, and $w_b(p_i)$ returns the probability of Boolean literal $p_i$. 
This implies that the weight function from (\ref{wmi}) can be factorized into a product of a real variable weight function $w_x$ and a product of individual Boolean literal weights as in  (\ref{fact}), and we refer to this as the \emph{factorization assumption}.

\begin{example}
Consider the formula $\phi_{\text{\DNF}}:$
\[
(p_1  \land (0 \leq x_1 \leq 5) \land \neg p_2) \lor (p_2 \land \neg (2 \leq x_1 \leq 4)).
\]
Let $0 \leq x_1 \leq 10$, and $w(\xbf, \vbf) = x_1 \cdot w_b(p_1) \cdot w_b(p_2)$, with $w_b(p_1) = 0.6$, and $w_b(p_2) = 0.1$. Hence, we obtain:
\begin{align*}
\begin{split}
\WMI(\phi_{\text{\DNF}}) =&~0.6\cdot(1-0.1)\cdot\int_0^5 x_1~dx_1~+~\\ &~0.1\cdot\Big(\int_0^2 x_1~dx_1 + \int_4^{10} x_1~dx_1\Big) = 11.15
\end{split}
\end{align*}

 \hfill \qedboxfull
\end{example}

Weighted model integration is defined on fragments of propositional logic in the obvious way, i.e., \WMI over \DNF formulas (resp. \CNF formulas) is the weighted model integration problem where the class of input formulas is restricted to formulas in \DNF (resp., \CNF).  We write $\WMI(\DNF)$ (resp., $\WMI(\CNF)$) to denote the specific problem.  

Finally, weighted model counting can be viewed as a special case of \WMI where $\phi$ is restricted to Boolean variables. Formally, the \emph{weighted model count} (\WMC) of $\phi$ is given by ${\sum_{\tau \models \phi} w(\tau)}$, where $w: \mathbb{B}^m \mapsto \Rbb^+$ is a \emph{weight function}.

\subsection{Approximations with Guarantees}
Model counting problems are \sharpP-hard to solve exactly, and thus are intractable for exact computation. As a result, techniques for efficient \emph{approximations} to model counting have been devised, a special class of which being \emph{fully polynomial randomized approximation schemes (FPRAS)}. Given a target error $0 < \epsilon < 1$ and confidence ${0< \delta < 1}$, an FPRAS computes an approximation $\hat{\mu}$ of the actual solution $\mu$, in polynomial time w.r.t the input, $\frac{1}{\epsilon}$, and $\frac{1}{\delta}$, such that %
\begin{align*}
{\Pr\big(\mu (1 - \epsilon) \leq \hat{\mu} \leq \mu (1 + \epsilon)\big) \geq 1 - \delta}.
\end{align*}

For \WMC on \DNF structures, the Karp, Luby, and Madras \shortcite{Karp89} algorithm (KLM), a special case of the Linear-Time Coverage (LTC) \cite{LubyThesis} algorithm, is an FPRAS. For a \DNF $\phi$ with $n$ variables and $m$ clauses, KLM runs ${T = 8(1+\epsilon) m \log({\frac{2}{\delta}})\frac{1}{\epsilon^2}}$ trials to compute a successful trial count $N$. At every trial, KLM performs the following: 
\begin{enumerate}[leftmargin=2\parindent]
\item If no current sample assignment $\tau$ exists, then a random clause $c_i$ is selected with probability $\Pr(c_i)/ \sum_{j=1}^m \Pr(c_j)$, where $\Pr(c_i) = \prod_{v \in c_i} w_b(v)$ . Afterwards, $\tau$ is sampled uniformly from the set of satisfying assignments for $c_i$. 
\item  Another clause $c_k$ (which could be identical to $c_i$) is \emph{uniformly} randomly sampled, and $\tau$ is checked against $c_k$. If $\tau \models c_k$, $N$ is incremented and  $\tau$ is re-sampled. Otherwise, $\tau$ is re-used in the next trial.
\end{enumerate}

KLM returns $T\sum_{j=1}^m p(c_j) / mN$ as an estimate for $\WMC(\DNF)$. Since assignment checking runs in $O(n)$, KLM thus runs in time $O\big(nm \epsilon^{-2}\log(\frac{1}{\delta})\big)$. 

For volume computation of convex bodies, Lov{\`a}sz and Vempala \shortcite{LovaszVol06} provide an FPRAS based on Multi-Phase Monte Carlo. In $n-$dimensional space, it uses $O^*(n)$ phases, where the asterisk denotes suppressed logarithmic factors, such that at every phase, random walk algorithms such as \emph{hit-and-run} \cite{ChenHitAndRun} are called over convex bodies at consecutive phases to approximate the ratios between their volumes. Since hit-and-run runs in $O^*(n^3)$, the overall volume computation runs in time $O^*(n^4)$.

\section{Weighted Model Integration over DNFs}
In this section, we propose \approxwmi, an algorithm for $\WMI(\DNF)$, and prove that it is an FPRAS. \approxwmi builds on work for approximately computing the volume of unions of convex bodies \cite{Bringmann}, which we introduce next.

\subsection{The Volume of Union of Convex Bodies}
\approxunion is an FPRAS for computing the volume of the union of convex bodies \cite{Bringmann}. More formally, given $k$ convex bodies $B_1, ..., B_k$, \approxunion returns an approximation of the volume of  $\bigcup_{i=1}^{k} B_i$, denoted $\vol (\bigcup_{i=1}^k B_i)$. This algorithm is based on the LTC algorithm, and extends it with approximate (``weak'') oracles in order to tackle the underlying volume computations.

\approxunion first computes the volume of every convex body approximately, with multiplicative error $\epsilon_v$, using an oracle \volumequery. 
Following this, it repeats the following procedure for $T$ trials to compute a successful trial count $N$. 
\begin{enumerate}[leftmargin=2\parindent]
\item If no sample point $p$ exists, \approxunion samples a body $B_i$ with probability  $\vol (B_i) / \sum_{j=1}^k \vol (B_j)$, and then approximately uniformly samples $p$ from this body using an oracle \samplequery with error $\epsilon_S$.
\item It then checks whether $p$ belongs to another uniformly chosen body $B'$ using another approximate oracle, \pointquery, with error $\epsilon_p$. If  $p \in B'$, a new $B_i$ and $p$ are sampled, and $N$ is incremented. Otherwise, $p$ is re-used in the next trial. 
\end{enumerate}

Following $T$ calls to \pointquery, the algorithm returns $T\sum_{j=1}^k \vol (B_j) / kN$ as an estimate of $\vol (\bigcup_{i=1}^k B_i)$. 

\approxunion is an FPRAS for the volume computation of a union of convex bodies under certain conditions, as stated in Theorem 2 of \cite{Bringmann}, and these conditions are satisfied with closed-form bounds from Lemma 3 of that work. All in all, the following result holds:
\begin{theorem}[Theorem 2 and Lemma 3, \cite{Bringmann}]
\label{thm: approxunion}
\approxunion relative to oracles \volumequery, \samplequery, \pointquery  with errors $\epsilon_V$, $\epsilon_S$, and $\epsilon_P$ respectively, is an FPRAS for $\vol (\bigcup_{i=1}^{k} B_i)$  with error $\epsilon$ and confidence $\frac{1}{4}$ using 
$T = \frac{24\ln(2)(1+\tilde{\epsilon})k}{\tilde{\epsilon}^2 - 8(\tilde{C}-1)k}$ iterations, with 
$\tilde{\epsilon} = \frac{\epsilon - \epsilon_V}{1 + \epsilon_V}$ and ${\tilde{C} = \frac{(1 + \epsilon_S)(1 + \epsilon_V)(1 + k\epsilon_P)}{(1 - \epsilon_V)(1 - \epsilon_P)}}$, for
$\epsilon_V,\, {\epsilon_S \leq \frac{\epsilon^2}{47k}}$, and $\epsilon_P \leq \frac{\epsilon^2}{47k^2}$.
\end{theorem}

Furthermore, when this theorem holds, T is $O(\frac{k}{\epsilon^2})$. \approxunion generalizes LTC to allow errors within sampling, membership checking, and volume computation, so long as these can be made arbitrarily small, and computes unions of continuous sets, as opposed to only discrete sets. 

\approxunion does \emph{not} allow for confidence parameters in the oracles, but  it can be extended to allow for FPRAS oracles having confidence $\delta$ %
using standard tools of probability.

\begin{lemma}
\label{lem:conf}
\approxunion relative to FPRAS oracles  \volumequery, \samplequery, and \pointquery  with errors $\epsilon_V$, $\epsilon_S$, $\epsilon_P$ and confidence values $\delta_V$, $\delta_S$, $\delta_P$, respectively, is an FPRAS with error $\epsilon$ and confidence $\delta$ for $\vol (\bigcup_{i=1}^{k} B_i)$ using $T = \frac{8\ln(\frac{8}{\delta})(1+\tilde{\epsilon})k}{\tilde{\epsilon}^2 - 8(\tilde{C}-1)k}$	iterations, for
	\begin{enumerate}[leftmargin=2\parindent]
		\item $\epsilon_V, \epsilon_S \leq \frac{\epsilon^2}{47k}$, $\epsilon_P \leq \frac{\epsilon^2}{47k^2}$, and 
		\item $\delta_V \leq \frac{\delta}{4k}$, $\delta_S + \delta_P \leq \frac{\delta}{2276 \ln(\frac{8}{\delta}) \frac{k}{\epsilon^2}}$.
	\end{enumerate}
\end{lemma}

\subsection{Approximating  the WMI over DNFs}
\label{ssec:indep}

\begin{algorithm}[tb]
	\caption{\approxwmi for $\WMI(\DNF)$.}
	\label{alg:wmi}
	\textbf{Input}: $\Xbf$: a set of $n$ real variables; $\Vbf$: a set of $m$ Boolean variables;  $\phi$: a \DNF  consisting of $k$ clauses $c_i$; $w$: a concave factorized weight function. \\
	\textbf{Parameters}: $\epsilon$: error, $\delta$: confidence. \\
	\textbf{Output}: $\WMI(\phi, w \mid \Xbf, \Vbf)$.
	\begin{algorithmic}[1] %
		\State $T \gets \frac{8\ln(\frac{8}{\delta})(1+\tilde{\epsilon})k}{\tilde{\epsilon}^2 - 8(\tilde{C}-1)k}$ \Comment $T$ is $O(\frac{k}{\epsilon^2} \ln(\frac{1}{\delta}))$  \label{line:1}
		\State $\delta_V \gets \frac{\delta}{4k}$, $\delta_S \gets \frac{\delta}{2276 \ln(\frac{8}{\delta}) \frac{k}{\epsilon^2}}$ \Comment Confidence parameters  \label{line:2}
		\For {$a \gets 1$ to $k$}   \label{line:3}
		\State $U_i \gets \clauseweight(c_i, w, \Xbf,\Vbf)[\frac{\epsilon^2}{47k},~\delta_V]$  \label{line:4}
		\EndFor 
		\State $U \gets \sum_{i=1}^m U_i$ \Comment{Sampling trials}
		\State time $ \gets 0$,~~ $N_T \gets 0$
		\While{$\text{time} < T$}  \label{line:7}
		\State Randomly select $c_i$, $ i \in {1, ..., k}$ with probability $\frac{U_i}{U}$   \label{line:8}
		\State 
		$p_\text{Bool}, p_\text{Real} \gets \sample(c_i, w, \Xbf, \Vbf)[\frac{\epsilon^2}{47k}, ~\delta_S]$    \label{line:9}
		
		\State $c_\text{sat} \gets \emph{false}$
		\While{$\neg c_\text{sat}$}
		\State Uniformly select $c_j$, where $j \in {1, ..., k}$   \label{line:12}
		\State $\text{time} \gets \text{time} + 1$
		\If{$\text{time} \geq T$} \Comment{If $T$ reached during trial}
		\State \textbf{return }$\sfrac{T U}{k N_T}$ 
		\EndIf
		\If{$\evaluate(c_j, p_\text{Bool}, p_\text{Real})$} \label{line:16}
		\State $c_\text{sat} \gets \emph{true}$, $N_T \gets N_T + 1$  
		\EndIf
		\EndWhile
		\EndWhile  \label{line:17}
		\State \textbf{return } $\sfrac{T U} {k N_T}$ \label{line:18}
	\end{algorithmic}
\end{algorithm}

\begin{algorithm}[tb]
	\caption{Subroutines of \approxwmi as functions \clauseweight, \sample, and \evaluate.}
	\label{alg:functions}
	\label{alg:oracles}
	\begin{algorithmic} %
		\Function{\clauseweight}{$c, w, \Xbf, \Vbf$}[$\epsilon, \delta$]
		\State $w_\text{Bool} \gets \prod_{p \in {c_i}} w_b(p)$ \Comment ($\text{c}_1$)
		\State Introduce a new variable $d$ in $\Xbf$ \Comment ($\text{c}_2$)
		\State $\xbf'_c \gets \xbf_c \land (0 \leq d \leq w_x(x))$
		\State $w_\text{Real} \gets \volume(\xbf'_c)[\epsilon,\delta]$ \Comment ($\text{c}_3$)
		\State \textbf{return } $w_\text{Bool}\cdot w_\text{Real}$ \Comment ($\text{c}_4$)
		\State
		\EndFunction
		\Function{\sample}{$c, w, \Xbf, \Vbf$}[$\epsilon, \delta$]
		\State $p_\text{Bool} \gets$ $b_c$, $\Vbf \setminus b_c$ sampled according to $w_b$ \Comment ($\text{s}_1$)
		\State Introduce a new variable $d$ in $\Xbf$ \Comment ($\text{s}_2$)
		\State $\xbf'_c \gets \xbf_c \land (0 \leq d \leq w_x(x))$
		\State $p_\text{Real}\gets \convexbodysampler(\xbf'_c)[\epsilon, \delta]$ \Comment ($\text{s}_3$)
		\State \textbf{return } $p_\text{Bool}$, $p_\text{Real}$ \Comment ($\text{s}_4$)
		\State 
		\EndFunction
		\Function{\evaluate}{$c$, $p_\text{Bool}$, $p_\text{Real}$}
		\State Compute convex polytope $\xbf_c$ from $c$
		\If{$p_\text{Real} \not\models \xbf_c$}\textbf{ return } \emph{false}  \Comment ($\text{e}_1$) 
		\EndIf
		\If{$p_\text{Bool}\not\models c$}  \textbf{ return } \emph{false} \Comment ($\text{e}_2$)
		\EndIf
		\State \textbf{return } \emph{true}
		\EndFunction
	\end{algorithmic}
\end{algorithm}
We can now introduce the algorithm \approxwmi (Algorithm~\ref{alg:wmi}) for $\WMI(\DNF)$, assuming $w$ is concave, i.e., $\forall x,y \in \mathbb{R}^n$, and  $\lambda \in [0,1]$,  it holds that
\[ 
\lambda w(x) + (1 - \lambda)w(y) \leq w(\lambda x + (1-\lambda)y).
\]
This assumption ensures that the bodies resulting from the application of weight functions on convex polytopes are also convex, which in turn enables the use of volume computation FPRAS algorithms \cite{LovaszVol06,KannanLS97} over these bodies. We also assume that $w$ uses the factorization assumption.

\paragraph{Overview of \approxwmi.}
Given a DNF $\phi$ over a hybrid domain $\Xbf \cup \Vbf$, and a concave weight function $w$ that factorizes, \approxwmi computes an $\epsilon, \delta$ approximation of $\WMI(\phi, w \mid \Xbf, \Vbf)$.
More specifically, \approxwmi extends the oracle functions of \approxunion in order to allow unreliable oracles (with confidence parameter $\delta$),  hybrid domains, and arbitrary factorized concave weight functions over convex bodies.%

The main steps of \approxwmi are as follows: After initializing the parameters (\ref{line:1}-\ref{line:2}), the first step is to compute the weighted model integral of the individual clauses in $\phi$, using the function \clauseweight (\ref{line:3}-\ref{line:4}). 
Then, the algorithm runs $T$ sampling trials to compute a successful trial count $N_T$  (\ref{line:7}-\ref{line:18}), similarly to LTC. 
In a sampling trial, a random clause $c$ is selected with probability proportional to its weight $U_i$, and then a point $p$ is sampled from $c$ according to $w$ using the function \sample (\ref{line:8}-\ref{line:9}). 
Afterwards,  a clause $c'$ (possibly $c$), is uniformly chosen from $\phi$ (\ref{line:12}), and a check is made via the function \evaluate, to verify the membership of $p$ to $c'$, and the estimator $N_T$ is incremented accordingly (\ref{line:16}-\ref{line:17}). 

We now explain the subroutines of \clauseweight, \sample, and \evaluate, given in Algorithm~\ref{alg:oracles}, in detail. 

\paragraph{\clauseweight.} This function returns an $\epsilon, \delta$ approximation of $\WMI(c, w \mid \Xbf, \Vbf)$ from a given conjunct $c$ and weight function $w$ over the domains $\Xbf, \Vbf$.
\clauseweight computes the product of probabilities for all Boolean literals, denoted by $p$, appearing in $c$  ($\text{c}_1$). It then transforms the polytope $\xbf_c$ defined by the LRA constraints in $c$, into ${\xbf'_c \in \mathbb{R}^{n+1}}$, by adding another constraint encoding the weight function, with this operation denoted by $\land$ ($\text{c}_2$). Thus, $\xbf'_c$ is identical to $\xbf_{c}$ across the first $n$ dimensions, with an added dimension $d$ verifying ${0 \leq d \leq w_x(x), x \in \xbf_{c}}$.
\clauseweight approximates the volume of $\xbf'_c$, as a proxy for computing the integral of $w$ over $\xbf_c$, using a convex body volume computation algorithm \cite{LovaszVol06,KannanLS97}, denoted by \volume ($\text{c}_3$). Since $w$ is concave, the added dimension in $\xbf'_c$ maintains the convexity of $\xbf_c$, and $\xbf'_c$ is therefore convex. Finally, \clauseweight returns the product of the steps ($\text{c}_1$) and ($\text{c}_3$) outputs, as its estimate for ${\WMI(c, w \mid \Xbf, \Vbf)}$ ($\text{c}_4$).

\paragraph{\sample.} 
This function samples a point $p$ over the domain $\Xbf \cup \Vbf$ from a given conjunct $c$ and a weight function $w$, as follows: It first samples Boolean assignments $p_\text{Bool}$ satisfying $c$ by setting all Boolean literals appearing in $c$ to their required values and randomly sampling all remaining variables according to $w_b$ ($\text{s}_1$). Then, it computes an analogous transformation from $\xbf_c$ to $\xbf'_c$ as in \clauseweight ($\text{s}_2$). 
Afterwards, \sample  samples a point $p_\text{Real}$ approximately uniformly from $\xbf'_c$, using standard sampling approaches for convex bodies such as hit-and-run \cite{ChenHitAndRun}, denoted by \convexbodysampler ($\text{s}_3$). It then discards the $n+1^{th}$ dimension to yield approximate samples $p_\text{Real}$ from $\xbf_c$ weighted according to $w$. 
Finally, \sample  ($\text{s}_4$) returns the concatenation of the outputs of ($\text{s}_1$) and ($\text{s}_3$) as a sample from $c$. Note that, since the weight factorization makes Boolean variable and real variable weights independent, $p_\text{Bool}$ and $p_\text{Real}$ can be sampled separately, as we outline here.

\paragraph{\evaluate.} 
This function determines the membership of a point $p \in \mathbb{R}^n \times \mathbb{B}^m$ to the body defined by a conjunct $c$. 
Specifically, it checks the membership of a point $p$ to the polytope defined by $c$ in two steps: \evaluate first verifies the Boolean component of $p$ ($\text{e}_1$), and then verifies the real component of $p$ (i.e., that all the LRA constraints of $c$ are satisfied) ($\text{e}_2$). If both conditions are met, then $p$ satisfies $c$. Unlike the earlier two functions, \evaluate is deterministic, and its outputs have no attached uncertainty.

\subsection{\approxwmi is an  FPRAS}
\label{ssec:approxwmi_fpras}
We show the correctness of \approxwmi, and prove that, under the error and confidence settings presented in Algorithm \ref{alg:wmi}, it is an FPRAS for \WMI(\DNF) with concave weight functions $w$, respecting the factorization assumption. 

\begin{theorem}
\label{thm: wmi_indep_fpras}
\approxwmi relative to FPRAS oracles \clauseweight, \sample, and \evaluate having error $\epsilon_V$, $\epsilon_S$, $\epsilon_P$ and confidence $\delta_V$, $\delta_S$, $\delta_P$, respectively, is an FPRAS for $\WMI(\DNF)$ over a concave and factorized weight function $w$,  with error $\epsilon$ and confidence $\delta$ and using $T = \frac{8\ln(\frac{8}{\delta})(1+\tilde{\epsilon})k}{\tilde{\epsilon}^2 - 8(\tilde{C}-1)k}$ iterations, for
\begin{enumerate}[leftmargin=2\parindent]
\item $\epsilon_V, \epsilon_S \leq \frac{\epsilon^2}{47k}$, $\epsilon_P \leq \frac{\epsilon^2}{47k^2}$, and 
\item $\delta_V \leq \frac{\delta}{4k}$, $\delta_S + \delta_P \leq \frac{\delta}{2276 \ln(\frac{8}{\delta}) \frac{k}{\epsilon^2}}$.
\end{enumerate}
\end{theorem}
\begin{proof}[Proof sketch]
\approxwmi is a variant of \approxunion, where the oracles are replaced with the specific oracles for \WMI. Thus, it suffices to show that all oracles in \approxwmi satisfy the conditions of Lemma \ref{lem:conf}. \evaluate is deterministic, so trivially satisfies the conditions. 
As for \clauseweight, we first verify that it correctly computes $\WMI(c, w \mid \Xbf, \Vbf)$.  We then show that \clauseweight meets the conditions of Lemma \ref{lem:conf}: as multiplication of Boolean weights is error-free, it is sufficient for \volume to have error $\epsilon \leq \epsilon_V$ and confidence $\delta \leq \delta_V$.
As for \sample,  we first show that sampling from the transformation result $\xbf'_c$ is equivalent to sampling from $\xbf_{c}$ according to $w$. Then, it suffices to run \convexbodysampler with parameters $\epsilon_S$ and $\delta_S$ to ensure \sample meets the requirements.
\end{proof}

The correctness of \approxwmi is clearly independent from the specific choice of oracles, and both for weak and strong oracles, the accuracy guarantees are preserved.  Assuming additionally that the oracles have an FPRAS, \approxwmi runs in polynomial time. More specifically, for \clauseweight runtime $r_c$, \sample runtime $r_s$, and \evaluate runtime $r_e$, \approxwmi runs in ${O(k \cdot r_c + T(r_s + r_e))}$. %

To illustrate, for the most general case of arbitrary concave weight functions and convex bodies, we can use the algorithm of \citeauthor{LovaszVol06} (\citeyear{LovaszVol06}), as the \volume oracle in \clauseweight, and hit-and-run for \convexbodysampler \cite{ChenHitAndRun}. These choices make \clauseweight run in time ${O^*(m+n^4(\frac{1}{\epsilon_V})^2)}$, where $m$ is the number of Boolean variables and $n$ is the number of real variables, and \sample run in time ${O^*(m + n^3(\frac{1}{\epsilon_S})^2)}$. Since \evaluate runs in deterministic polynomial time, namely $O(m + Wn)$, where $W$ is the width of a conjunction $c$, \approxwmi therefore runs in time: 
\begin{align*}
\begin{split}
&O^*\Big(km+kn^4(\frac{1}{\epsilon_V})^2 + Tm + Tn^3(\frac{1}{\epsilon_S})^2 +T(m +Wn)\Big)\\
=~&O^*\Big(km + \frac{k^3}{\epsilon^4}n^4 + \frac{k}{\epsilon^2}m + \frac{k^3}{\epsilon^6}n^3 + \frac{k(Wn + m)}{\epsilon^2}\Big)\\
=~&O^*\Big(\frac{k^3}{\epsilon^4}n^4 + \frac{k^3}{\epsilon^6}n^3 + \frac{k}{\epsilon^2}(m + Wn)\Big).\\
\end{split}
\end{align*}

Note that the time complexity of \sample can be reduced by restricting the class of bodies and weight functions used. Indeed, if $w$ is restricted to be linear, e.g., $3x_1 + 2x_2 - x_4$, then all bodies can be sampled approximately using optimized polytope sampling methods such as geodesic walks \cite{LeeV17}, which can run significantly faster for bodies defined with a small number of LRA constraints. Further to this, box-shaped bodies, defined by using no more than one variable per constraint, paired with a constant $w$, i.e., a uniform distribution over the problem domain, are an instance of unweighted model integration, and can be trivially sampled using uniform sampling. Nonetheless, we assume the general case of convex bodies in this work, and build our algorithm accordingly.  

\subsection{Extending \approxwmi: \approxwmidep}
\label{ssec:extension}

In Section \ref{ssec:indep}, we presented \approxwmi, an FPRAS for \WMI over \DNF formulas with factorized and concave weight function $w$. The factorization of $w$ simplifies \WMI, in that (i) it makes Boolean variable weights independent from real variables, and vice-versa, and (ii) simplifies the joint distribution over Booleans to a product of  weights. This standard factorization prevents any changes in the Boolean domain from affecting $w_x$, but can be used to capture weight functions with this behaviour by defining mutually exclusive Boolean partitions of the problem domain, in which no dependencies between Boolean and real variables exist, and then applying the factorization separately over each of these partitions \cite{Martires19}. However, the number of such partitions can be exponential in the worst-case, which makes this approach intractable in practice. 

In this section, we introduce a more general factorization for $w$, such that 
\begin{align}
\label{eq:newfact}
w(\xbf, \vbf) = w_x(\xbf,\vbf) \prod_{i=1}^m w_b(p_i),
\end{align} where  $w_x$ also depends on Boolean variables. We then propose an FPRAS \approxwmidep for this factorization. \approxwmidep extends \approxwmi,  in particular its oracle \clauseweight, to function under this more general factorization.%

We now describe how the more general weight function $w$ is defined.
Let $\xbf_c \subseteq \mathbb{R}^n$ be the polytope defined by the LRA constraints of a conjunct $c$, and $v \in \mathbb{B}^m$ be a Boolean assignment. 
Let $a, b: P(\Rbb^n) \mapsto \Rbb^+$, where $P$ denotes the power set operation, be functions such that 
\begin{align*}
a(\xbf_c) &= \min_\vbf \int\limits_{\xbf_c}  w_x(\ybf, \vbf)~d \ybf, \text{ and}\\
b(\xbf_c) &= \max_\vbf \int\limits_{\xbf_c}  w_x(\ybf, \vbf)~d \ybf.
\end{align*}
These functions are guaranteed to exist and are finite as the Boolean domain $\mathbb{B}^m$ is finite and all integral computations in the scope of \WMI are finite-valued. Hence, we note that $\forall \xbf_{c} \in P(\mathbb{R}^n), \forall \vbf \in \mathbb{B}^m,\exists a,b:  P(\Rbb^n) \mapsto \Rbb^+, \text{such that }$
\begin{align}
\label{eq:existence}
 a(\xbf_c) \leq \int\limits_{\xbf_c}  w_x(\ybf, \vbf)~d \ybf~\leq~ b(\xbf_{c}).
\end{align}
Let us define $\rho = \max_{\xbf_c} \frac{b(\xbf_{c})}{a(\xbf_{c})}$. In what follows, we restrict the choice of $w$ such that, for all $\xbf_c$, $\rho$ is upper-bounded by a polynomial in $\frac{1}{\epsilon_\text{Samp}}$, $\frac{1}{\delta_\text{Samp}}$, $n$, $m$, and $k$. 
In this case, we say that a weight function $w$ is \emph{$\rho$-restricted}.

Intuitively, \emph{$\rho$-restriction} ensures that integrals over the same body, computed with different Boolean instantiations of the weight function, yield results that are within a tractable ratio from one another, which in turn allows the efficient use of sampling techniques. 

While this restriction is similar in nature to restrictions on \emph{tilt} $\theta$ in existing works on weighted model counting (see e.g.~\cite{ChakrabortyFMSV14}), it is not identical to them. More specifically, tilt is the ratio between the maximum weight of a satisfying assignment and the minimum weight of a satisfying assignment, whereas $\rho$ is the ratio between the maximum integral and minimum integral of a weight function over any given real body. In fact, restrictions on $\rho$ are looser than restrictions on $\theta$. Indeed, when $\theta$ is bounded by a value $H$, it is simple to show that $\rho$ is also bounded by $H$, whereas the same cannot be said in the opposite direction.  

Consider a simple example, with one Boolean variable $v$ and one real variable $x$, such that $0 \leq x \leq 10$, and let  
\[w_x(x, v) = \begin{cases} 
\frac{e^{10}}{1 + e^{-x}} - 0.5e^{10} + 1 & \text{if} \ v \\ \frac{2e^{10}}{1 + e^{-x}} - e^{10} + 2 & \text{otherwise,} \end{cases}\]
In this example, $\rho$ is clearly upper-bounded by $2$, but $\theta$ is upper-bounded by $2+ e^{10}$. 
To upper-bound $\rho$ easily in practice, it is sufficient to upper-bound, for all real assignments $x$, the ratio between the maximum $w(\xbf,\vbf)$ and the minimum $w(\xbf,\vbf)$ over all Boolean assignments, or more formally, $w_{\max}(\xbf) / w_{\min}(\xbf)$, where $w_{\min}(\xbf) = \min_{\vbf} w(\xbf,\vbf)$ and $w_{\max}(\xbf) = \max_{\vbf} w(\xbf,\vbf)$. %

\approxwmidep extends \approxwmi to also handle weight functions $w$, where $w_x$ may depend on Boolean variables. To achieve this, \approxwmidep replaces the oracles \clauseweight and \sample used in Algorithm \ref{alg:wmi}, with \clauseweightdep and \sampledep, respectively (while other details remain unaffected). Hence, we present these extended oracles \clauseweightdep and \sampledep, presented in Algorithm \ref{alg:functions_dep}, in more detail.

\begin{algorithm}[tb]
	\caption{Subroutines of \approxwmidep as functions \clauseweightdep and \sampledep.}
	\label{alg:functions_dep}
	\begin{algorithmic}[] %
		\Function{\clauseweightdep}{$c, w, \Xbf, \Vbf$}[$\epsilon, \delta$]
		\State  $\epsilon_\text{Samp}, \epsilon_\text{Comp} \gets \frac{\epsilon}{1 + \sqrt{2}}$, $\delta_\text{Samp} \gets \frac{\delta}{2}$
		\State $s \gets \ln(\frac{2}{\delta_\text{Samp}})\frac{1}{\epsilon_\text{Samp}^2}\rho^2$ \Comment \# of sampling trials
		\State $\delta_\text{Comp} \gets \frac{\delta}{2s}$ 
		\For {$i \gets 1$ to $s$}
		\State $\tau_i \gets b_c$, $\Vbf \setminus b_c$ sampled according to $w_b$ \Comment ($\text{c}_1$)
\State Introduce a new variable $d$ in $\Xbf$ \Comment ($\text{c}_2$)
		\State $\xbf'_c \gets \xbf_c \land (0 \leq d \leq  w_x(\xbf, \tau_i))$
		\State $X_i \gets \volume_{\epsilon_{\text{Comp}},\delta_\text{Comp}}(\xbf'_{c})$ \Comment ($\text{c}_3$)
		\EndFor 
		\State $\text{BoolWeight} \gets \prod_{p \in {c_i}} w_b(p)$ \Comment ($\text{c}_4$)
		\vspace{0.1cm}
		\State \textbf{return } $\text{BoolWeight} \cdot \frac{1}{s}\sum_{i=1}^{s}X_i$ \Comment ($\text{c}_5$)
		\EndFunction
		\Function{\sampledep}{$c, w, \Xbf, \Vbf$}[$\epsilon, \delta$]
		\State $p_\text{Bool} \gets$ $b_c$, $\Vbf \setminus b_c$ sampled according to $w_b$ \Comment ($\text{s}_1$)
		\State Introduce a new variable $d$ in $\Xbf$ \Comment ($\text{s}_2$)
		\State $\xbf'_c \gets \xbf_c \land (0 \leq d \leq w_x(\xbf, p_\text{Bool}))$
		\State $p_\text{Real}\gets \convexbodysampler(\xbf'_c)[\epsilon, \delta]$ \Comment ($\text{s}_3$)
		\State \textbf{return } $p_\text{Bool}$, $p_\text{Real}$ \Comment ($\text{s}_4$)
		\State 
		\EndFunction
	\end{algorithmic}
\end{algorithm}

\paragraph{\clauseweightdep.} \clauseweightdep performs  sampling over the set of Boolean assignments to estimate the \WMI of $c$, and this sampling occurs with  error $\epsilon_\text{Samp}$ and confidence $\delta_\text{Samp}$. The number $s$ of sampling trials is a function of $\epsilon_\text{Samp}$ and $\delta_\text{Samp}$, as well as $\rho$, the integral ratio defined earlier. Within every trial, \volume is called with error $\epsilon_\text{Comp}$ and confidence $\delta_\text{Comp}$. 

More specifically, given a conjunct $c$ and a weight function $w$, \clauseweightdep performs $s$ \emph{sampling rounds} to estimate $\WMI(c, w~|~\Xbf, \Vbf)$, where each sampling round consists of randomly sampling a Boolean assignment $\tau$ satisfying $c$ according to $w_b$ ($\text{c}_1$), computing the induced weight function $w_x(x, \tau)$ and using it to obtain the transformed convex body $\xbf'_c$ ($\text{c}_2$), and computing the weighted integral over $c$ using \volume ($\text{c}_3$). At the end of the $s$ sampling steps, the product of all $b_c$ literal weights is computed ($\text{c}_4$) , and the function returns this product, multiplied by the average of all sampling results, as an approximation of $\WMI(c, w~|~\Xbf, \Vbf)$ ($\text{c}_5$).

\paragraph{\sampledep.} This function is defined almost identically to \sample in \approxwmi, with the only minor difference being that $w_x(\xbf, p_\text{Bool})$ must be induced from $p_\text{Bool}$ in \sampledep($\text{s}_2$) prior to applying the transformation. This is because $w_x$ also depends on Boolean variables in this setting. Unlike \sample, where Boolean and real variable sampling can be done in any order, it is necessary for Boolean sampling to run first in \sampledep, so as to condition $w_x$ on the Boolean sample output and subsequently sample from $\xbf_c$ according to the induced weight function $w_x(\xbf, p_\text{Bool})$.

\subsection{\approxwmidep is an FPRAS}
We show that \approxwmidep is an FPRAS for $\WMI(\DNF)$, by lifting the result given in Theorem~\ref{thm: wmi_indep_fpras}.
To do so, we first need to show the correctness of \clauseweightdep, and prove that it is an FPRAS for $\WMI(c, w~|~\Xbf, \Vbf)$ over concave, $\rho$-restricted $w$ factorized using this more general factorization.  

\begin{lemma}
\label{lem:voldep}
	\clauseweightdep relative to Monte-Carlo sampling error $\epsilon_\text{Samp}$ and confidence $\delta_\text{Samp}$, and an FPRAS \volume with error $\epsilon_\text{Comp}$ and confidence $\delta_\text{Comp}$, is an FPRAS for $\WMI(c, w | \Xbf, \Vbf)$, where $c$ is a clause and $w$ is concave, $\rho$-restricted, and factorized according to Equation \ref{eq:newfact}, with error $\epsilon$, confidence $\delta$, and using $s = \ln(\frac{2}{\delta_\text{Samp}})\frac{1}{\epsilon_\text{Samp}^2}\rho^2$ iterations, for 	\begin{enumerate}[leftmargin=2\parindent]
		\item $\epsilon_\text{Samp}, \epsilon_\text{Comp} \leq \frac{\epsilon}{1 + \sqrt{2}}$,
		\item $\delta_\text{Samp} \leq \frac{\delta}{2}$, and
		\item $\delta_\text{Comp} \leq \frac{\delta}{2s}.$ 
	\end{enumerate}
\end{lemma}

\begin{proof}[Proof sketch]
Under this more general factorization, \WMI for a conjunct $c$ can be computed as a sum of integrals over the real polytope $\xbf_c$ given all (possibly exponential) possible weight functions induced by Boolean assignments. Monte-Carlo sampling approximates $\WMI(c, w \mid \Xbf, \Vbf)$, avoids the exponential blow-up, and can provide guarantees, since $w$ is $\rho$-restricted. %
Clearly, the results of every sampling step are $s$ independent and identically distributed (i.i.d) random variables which, from Equation \ref{eq:existence}, are bounded by $a(\xbf_c)$ and $b(\xbf_c)$. Applying the Hoeffding bound with the target additive difference a $\epsilon_\text{Samp}$ multiple of the expected \WMI yields the lower bound for $s$ shown in Algorithm \ref{alg:functions_dep}. \clauseweightdep therefore also runs in polynomial time with respect to $\frac{1}{\epsilon_\text{Samp}}$, $\frac{1}{\delta_\text{Samp}}$, $n$, $m$,  and  $k$ %
, since a sampling iteration runs in polynomial time (\volume is an FPRAS), and $s$ is polynomial given that $w$ is $\rho$-restricted.

It now remains to show that, under the conditions of Lemma \ref{lem:voldep}, \clauseweightdep produces an $\epsilon, \delta$ approximation of $\WMI(c, w | \Xbf, \Vbf)$. First, we prove that, when no failure occurs within \clauseweightdep (i.e., all oracles and the sampling procedure respect their error bounds), the provided error bounds in Lemma \ref{lem:voldep} produce an estimate within the overall $\epsilon$ error requirement, and the multiplication of the two $\frac{\epsilon}{1 + \sqrt{2}}$ bounds for sampling and volume computation, combined with $\epsilon \leq 1$, yields the desired result. Second, we show that, with the asserted confidence bounds, the probability of any failure in \clauseweightdep, is upper-bounded by $\delta$, using the union bound.
\end{proof}
	
We can now combine Theorem \ref{thm: wmi_indep_fpras} and Lemma \ref{lem:voldep} to show that \approxwmidep is an FPRAS for \WMI(\DNF) for this more general factorization of $w$, and given a concave and $\rho$-restricted $w$. 

\begin{theorem}
\label{thm:approxdep}
	\approxwmidep,  
relative to FPRAS oracles \clauseweightdep, \sampledep, and \evaluate having error $\epsilon_V$, $\epsilon_S$, $\epsilon_P$ and confidence $\delta_V$, $\delta_S$, $\delta_P$, respectively, is an FPRAS for $\WMI(\DNF)$ over a $w$ that is concave, $\rho$-restricted, and factorized according to Equation \ref{eq:newfact}, with error $\epsilon$ and confidence $\delta$ and using $T = \frac{8\ln(\frac{8}{\delta})(1+\tilde{\epsilon})k}{\tilde{\epsilon}^2 - 8(\tilde{C}-1)k}$ iterations, for   
\begin{enumerate}[leftmargin=2\parindent]
\item $\epsilon_V, \epsilon_S \leq \frac{\epsilon^2}{47k}$, $\epsilon_P \leq \frac{\epsilon^2}{47k^2}$, and 
\item $\delta_V \leq \frac{\delta}{4k}$, $\delta_S + \delta_P \leq \frac{\delta}{2276 \ln(\frac{8}{\delta}) \frac{k}{\epsilon^2}}$.
\end{enumerate}
\end{theorem}

Finally, the running time of \approxwmidep, as a function of  \clauseweightdep runtime $r'_c$, \sampledep runtime $r'_s$, and \evaluate runtime $r_e$, is ${O(k \cdot r'_c + T(r'_s + r_e))}$, analogously to \approxwmi. Furthermore, for the same choice of \convexbodysampler, it is easy to see that $r_s = r'_s$. However, the main difference in running time between \approxwmi and \approxwmidep comes from the difference between $r'_c$ and $r_c$. Indeed, given the same choice of \volume oracle, with running time $r_v$, $r_c = O(m + r_v)$, whereas 
$
r'_c = O(s(m + r_v)) = O(\frac{1}{\epsilon^4} \ln(\frac{1}{\delta})(m + r_v)).
$
Hence, the wider applicability of \approxwmidep comes at the expense of a larger runtime complexity, owing to the larger power of $\frac{1}{\epsilon}$ required for \clauseweightdep.
\begin{figure*}[!ht]
\begin{tikzpicture}[scale=0.64]

\definecolor{color0}{rgb}{0.12156862745098,0.466666666666667,0.705882352941177}
\definecolor{color1}{rgb}{1,0.498039215686275,0.0549019607843137}
\definecolor{color2}{rgb}{0.172549019607843,0.627450980392157,0.172549019607843}
\definecolor{color3}{rgb}{0.83921568627451,0.152941176470588,0.156862745098039}

\begin{axis}[
legend cell align={left},
legend entries={{$W=3$},{$W=5$},{$W=8$},{$W=13$}},
legend style={at={(0.01,0.99)}, font=\tiny, line width=0.4mm, anchor=north west, draw=white!80.0!black},
tick align=outside,
tick pos=left,
title={$\epsilon=0.15, \delta=0.05$},
x grid style={white!69.01960784313725!black},
xlabel={Number of variables ($m+n$)},
xmajorgrids,
xmin=50, xmax=1000,
y grid style={white!69.01960784313725!black},
y label style={at={(axis description cs:-0.00,.5)},anchor=south},
x label style={at={(axis description cs:0.5,-0.02)},anchor=north},
ylabel={Execution time (s)},
ymajorgrids,
ymin=0, ymax=5000
]
\addlegendimage{no markers, color0}
\addlegendimage{no markers, color1}
\addlegendimage{no markers, color2}
\addlegendimage{no markers, color3}
\addplot [line width=0.4mm, color0]
table [row sep=\\]{%
100	673.456 \\
200	2035.2955 \\
300	3241.0245 \\
400	4941.1244 \\
};
\addplot [line width=0.4mm, color1]
table [row sep=\\]{%
100	127.792 \\
200	307.363 \\
300	685.7165 \\
400	1627.0975 \\
500	1961.358 \\
600	3206.826 \\
700	4241.93 \\
};
\addplot [line width=0.4mm, color2]
table [row sep=\\]{%
100	122.493 \\
200	103.0015 \\
300	249.0355 \\
400	352.359 \\
500	489.279 \\
600	686.049 \\
700	827.5085 \\
800	1380.323 \\
900	1585.0565 \\
1000	2042.9885 \\
};
\addplot [line width=0.4mm, color3]
table [row sep=\\]{%
100	45.5175 \\
200	158.312 \\
300	139.8635 \\
400	218.323 \\
500	320.593 \\
600	422.9295 \\
700	452.723 \\
800	527.8875 \\
900	779.919 \\
1000	946.253 \\
};
\end{axis}
\end{tikzpicture}
\begin{tikzpicture}[scale=0.64]

\definecolor{color0}{rgb}{0.12156862745098,0.466666666666667,0.705882352941177}
\definecolor{color1}{rgb}{1,0.498039215686275,0.0549019607843137}
\definecolor{color2}{rgb}{0.172549019607843,0.627450980392157,0.172549019607843}
\definecolor{color3}{rgb}{0.83921568627451,0.152941176470588,0.156862745098039}

\begin{axis}[
legend cell align={left},
legend entries={{$W=3$},{$W=5$},{$W=8$},{$W=13$}},
legend style={at={(0.02,0.98)}, font=\tiny, line width=0.4mm,anchor=north west, draw=white!80.0!black},
tick align=outside,
tick pos=left,
title={$\epsilon=0.25, \delta=0.15$},
x grid style={white!69.01960784313725!black},
xlabel={Number of variables ($m+n$)},
xmajorgrids,
xmin=50, xmax=1000,
y grid style={white!69.01960784313725!black},
y label style={at={(axis description cs:-0.00,.5)},anchor=south},
x label style={at={(axis description cs:0.5,-0.02)},anchor=north},
ylabel={Execution time (s)},
ymajorgrids,
ymin=0, ymax=5000
]
\addlegendimage{no markers, color0}
\addlegendimage{no markers, color1}
\addlegendimage{no markers, color2}
\addlegendimage{no markers, color3}
\addplot [line width=0.4mm, color0]
table [row sep=\\]{%
100	70.3298267668378 \\
200	269.133403299917 \\
300	517.992043100702 \\
400	869.87811037619 \\
500	1449.93498609445 \\
600	2159.2811838649 \\
800	4542.12449892148 \\
};
\addplot [line width=0.4mm, color1]
table [row sep=\\]{%
100	26.9806045541402 \\
200	50.0343838108165 \\
300	109.419076403219 \\
400	255.210158718536 \\
500	253.016979257567 \\
600	630.553504063634 \\
700	820.713695247077 \\
800	1220.80788287187 \\
900	1407.18087385479 \\
1000	2026.65739228447 \\
};
\addplot [line width=0.4mm, color2]
table [row sep=\\]{%
100	15.9341998727557 \\
200	46.2906157279371 \\
300	57.0361323009228 \\
400	110.436019953926 \\
500	132.533255535635 \\
600	244.995750622471 \\
700	423.182070802101 \\
800	541.559774340906 \\
900	595.02395932924 \\
1000	833.492430644802 \\
};
\addplot [line width=0.4mm, color3]
table [row sep=\\]{%
100	22.2194167732746 \\
200	40.0061639975685 \\
300	44.3505420210577 \\
400	67.0665132911273 \\
500	88.3408451406458 \\
600	125.483281940778 \\
700	166.572419166643 \\
800	184.334991093523 \\
900	365.201851145921 \\
1000	483.042941039443 \\
};
\end{axis}

\end{tikzpicture}
\begin{tikzpicture}[scale=0.64]

\definecolor{color0}{rgb}{0.12156862745098,0.466666666666667,0.705882352941177}
\definecolor{color1}{rgb}{1,0.498039215686275,0.0549019607843137}
\definecolor{color2}{rgb}{0.172549019607843,0.627450980392157,0.172549019607843}
\definecolor{color3}{rgb}{0.83921568627451,0.152941176470588,0.156862745098039}

\begin{axis}[
legend cell align={left},
legend entries={{$W=3$},{$W=5$},{$W=8$},{$W=13$}},
legend style={at={(0.02,0.98)}, line width=0.4mm,font=\tiny, anchor=north west, draw=white!80.0!black},
tick align=outside,
tick pos=left,
title={$\epsilon=0.35, \delta=0.25$},
x grid style={white!69.01960784313725!black},
xlabel={Number of variables ($m+n$)},
xmajorgrids,
xmin=50, xmax=1000,
y grid style={white!69.01960784313725!black},
y label style={at={(axis description cs:-0.00,.5)},anchor=south},
x label style={at={(axis description cs:0.5,-0.02)},anchor=north},
ylabel={Execution time (s)},
ymajorgrids,
ymin=0, ymax=5000
]
\addlegendimage{no markers, color0}
\addlegendimage{no markers, color1}
\addlegendimage{no markers, color2}
\addlegendimage{no markers, color3}
\addplot [line width=0.4mm, color0]
table [row sep=\\]{%
50	28.63 \\
100	65.345 \\
200	181.909 \\
300	341.614 \\
400	793.5755 \\
500	1126.8665 \\
600	1614.9835 \\
700	2281.744 \\
800	2777.89 \\
900	3793.814 \\
1000	4897.800812347 \\
};
\addplot [line width=0.4mm, color1]
table [row sep=\\]{%
50	19.9017 \\
100	29.2761 \\
200	59.68305 \\
300	103.5279 \\
400	212.59665 \\
500	277.0488 \\
600	562.1598 \\
700	772.7229 \\
800	1131.4197 \\
900	1239.09615 \\
1000	1594.8036 \\
};
\addplot [line width=0.4mm, color2]
table [row sep=\\]{%
50	17.2536 \\
100	21.21 \\
200	34.4001 \\
300	54.39 \\
400	89.7288 \\
500	111.8187 \\
600	170.1273 \\
700	204.9936 \\
800	465.6117 \\
900	449.6559 \\
1000	642.0915 \\
};
\addplot [line width=0.4mm, color3]
table [row sep=\\]{%
50	9.81575 \\
100	22.35625 \\
200	39.1825 \\
300	36.25475 \\
400	63.931 \\
500	71.33525 \\
600	108.54375 \\
700	126.2205 \\
800	177.43775 \\
900	270.3985 \\
1000	380.24675 \\
};
\end{axis}

\end{tikzpicture}

	\caption{Runtime results for \approxwmi relative to different $\epsilon,\delta$. For $\epsilon=0.35, \delta=0.25$, all instances, including those with 1K variables, terminate within 5000 seconds, and higher-width instances ($W=8,13$) all terminate within ~2000 seconds for all $\epsilon, \delta$.}
	\label{fig:exps}
\end{figure*}
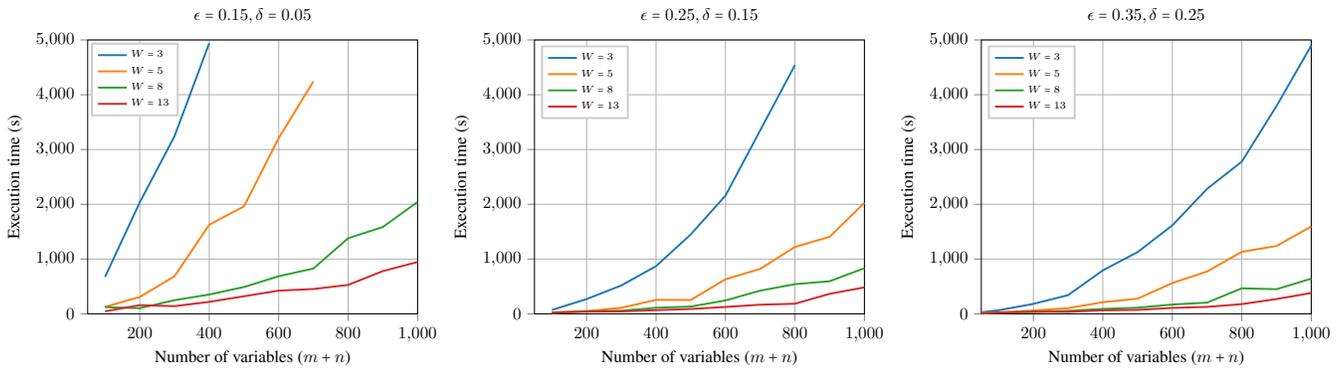
\section{Experimental Evaluation}
To evaluate the performance of \approxwmi, we generate random \DNF formulas and measure the time \approxwmi requires to solve them. We explain data generation and experimental setup in the following subsections.
\subsection{Generating Evaluation Data}
To evaluate \approxwmi, we generate \DNF formulas with total number of variables between $100$ and $1K$ inclusive, in increments of 100. Variables are split equally between real and Boolean (i.e., $m=n$).  Clause width $W$ is fixed between 3, 5, 8, 13, and the number of clauses is $k = \lfloor\frac{m+n+20}{W}\rfloor$. For every configuration ($m$, $n$, $k$, $W$), we generate 4 formulas, resulting in a total of 176 formulas used for evaluation.

Given a configuration, we generate a propositional DNF with $m+n$ variables. This DNF has $k \cdot W$ ``slots'', corresponding to the vacant literal positions to be filled in its clauses, and these slots are allocated to variables such that every variable appears at least once in the DNF. 

With probability 0.5, this allocation occurs uniformly. However, when this isn't the case, a ``privileged mechanism'' is used, such that a randomly selected small subset of ``privileged'' variables is allocated significantly more slots than the remaining variables (to encourage more dependencies across clauses), thereby giving these variables a larger impact on the formula \WMI. Once all variables are allocated, their literal sign is chosen uniformly at random. This data generation is also used and presented in more detail earlier \cite{ACL-AAAI20}.

Once a propositional DNF formula is generated, LRA constraints are then incorporated as follows: First, the slots for $n$ variables, corresponding to variable indices $m+1$ to $m+n$, are all replaced with LRA constraints, which in turn are generated at the clause level. For a conjunctive clause $c$ with $q$ LRA constraints to be generated:
\begin{enumerate}
	\item A random point $p$ in the real domain is uniformly sampled such that, by construction, $p \models c$, so that the real polytope defined by LRA constraints in $c$ is non-empty. 
	\item Generate $q$ constraints by
	(i) randomly selecting a subset $S$ of $Geom(1/L)$ real variables, where $Geom$ denotes the geometric distribution and $L=2$,
	(ii) randomly sampling $w_s \in \mathbb{R}^{|S|}$ weights for these variables and (iii) generating a random value $v$ and setting the linear constraint ${w.S \leq v}$.
	\item If $p$ satisfies $w.S \leq v$, then $w.S \leq v$ is added to $c$, otherwise, $w.S \geq v$, which $p$ must then satisfy, is added.
\end{enumerate}
\subsection{Experimental Setup}
In our experiments, we bound all real variables such that $\forall x \in \Xbf, 0 \leq x \leq 10$, to ensure finite integrals. We then evaluate \approxwmi with \emph{polynomial} weight functions: $w$ is a sum of up to 4 polynomial terms, each with a random constant weight, and degree randomly chosen using $Geom(0.6)$ and upper-bounded by 5. For terms with degree of 2 or higher, the constant weight is constrained to be negative to maintain concavity, e.g. $w = 20 - 3x_1^2x_4 - 2.2x_3^4 + x_2$.

We run \approxwmi using LattE \cite{baldoni2014user} for \volume  within \clauseweight as, despite being an exact solver, it supports polynomial weight functions while performing reliably in practice for smaller-scale formulas compared with approximate techniques \cite{EmirisF18}. 
We optimize our use of LattE by separately (trivially) integrating over variables not appearing in a clause or a term of $w$, and only running LattE over the appearing variables. This optimization is effective as (i) the number of real variables appearing in a clause is small in expectation (at worst $LW$) and (ii) the DNF formula structure breaks down the $n$-dimensional integration task into $k$ smaller parts which can be solved more efficiently (unlike for CNF). 
For sampling, we use hit-and-run \cite{ChenHitAndRun} for \convexbodysampler within \sample, with a constant factor $10^4$ used to compute the number of  walk iterations, as is standard in practice. It is \emph{unknown} whether this constant preserves theoretical guarantees, but it is widely used given that the best-known theoretical constant factors for hit-and-run are loose and prohibitively large (i.e., $10^{30}$) \cite{lovasz2003start}.

Finally, results were averaged over 5 runs with 3 $(\epsilon, \delta)$ settings: (0.15, 0.05), (0.25, 0.15), and (0.35, 0.25). Experiments ran with a timeout of 5000 seconds on a server with a Haswell 5-2640v3, 2.60GHz CPU and 12 GB of RAM. 

\subsection{Experimental Results}
All \approxwmi running times w.r.t. $m+n$, $W$, $\epsilon$, and $\delta$, are presented in Figure \ref{fig:exps}: \approxwmi performs very encouragingly, and solves \DNF instances with up to $1K$ variables within 5000 seconds over all clause widths $W$ for $\epsilon=0.35$ and $\delta=0.25$. In fact, instances with $1K$ variables and  $W=5,8, 13$ all run within 1600 seconds. For tighter $\epsilon$ and $\delta$, the system maintains high performance, despite the high power of $\epsilon^{-1}$ in the running time of \approxwmi due to \sample (cf. Section \ref{ssec:approxwmi_fpras}). Indeed, even with $\epsilon=0.15$ and $\delta=0.05$, all instances with widths 8 and 13 finish within $\sim$2000 seconds, whereas large instances of width 3 and $m+n \geq 500$, and width 5, $m+n \geq 800$, time out.

Somewhat unintuitively, system performance worsens as $W$ \emph{decreases}. For \DNF instances with $W=3$, \approxwmi requires almost triple the time compared to instances with higher $W$. Though this behavior is surprising, it can be attributed to an increased number of sampling replacements within \approxwmi. Indeed, for smaller widths, it is likelier that a call to \evaluate will yield ``True'' since there are less constraints to satisfy. Hence, further calls to \sample, which runs in $O^*(m + n^3)$, will be required, imposing a significant computational overhead on \approxwmi. This behavior justifies improved performance with increased width, as the expected number of calls to \sample decreases. 

These results confirm our intuitions about \approxwmi. First, they highlight that \approxwmi can indeed scale to large instances having up to 1K total variables, even with tight $\epsilon$ and $\delta$. Second, they show that \clauseweight calls are not a bottleneck, as they all run in less than 450 seconds in any instance,  due to the bounded width $W$ used in our experiments. Third, our results show that the main performance bottleneck for \approxwmi is the number of calls to \sample. This is particularly evident from the high dependence of running time on width $W$, and shows that, even with a reduced constant factor, convex body sampling remains a highly costly operation. 

The use of an exact \clauseweight oracle allows significant gains, owing to a reduced error requirement for \sample. Indeed, given exact \clauseweight and \evaluate, i.e., $\epsilon_V, \epsilon_P, \delta_V, \delta_P = 0$, Theorem \ref{thm: approxunion} and the union bound yield ${\epsilon_S < \frac{\epsilon^2}{8k}}$ and ${\delta_S \leq \frac{\delta}{1518 \ln(\frac{8}{\delta}) \frac{k}{\epsilon^2}}}$ respectively, and these looser bounds for sampling reduce the running time of \approxwmi significantly.
Overall, \approxwmi scales to \DNF instances with up to 1K variables using standard oracle implementations. This is particularly true for larger widths, as the number of \sample calls decreases. 

\section{Related Work}
\WMC is a unifying tool for probabilistic inference.~Inference in probabilistic graphical models~\cite{Koller-PGM} reduces to $\WMC(\CNF)$~\cite{SBK-05,CD08}. Similar reductions exist for Markov Logic Networks (MLNs) ~\cite{Richardson2006}, probabilistic logic programming~\cite{PPWMC}, and more generally, for relational models~\cite{GD16}. ~Still, \WMC cannot capture hybrid probabilistic models that are extensively studied; see e.g.~hybrid MLNs~\cite{WD08}, and hybrid Bayesian networks~\cite{GD05,SA12}. \WMI is proposed as a unifying inference tool in these hybrid models~\cite{Belle15}.

Weighted model integration/counting is \sharpP-hard \cite{Valiant79}, so is highly intractable. Nonetheless, many general-purpose exact solvers, based on several optimizations, have been developed. For instance,  \citeauthor{MorettinAIJ19}~(\citeyear{MorettinAIJ19}) propose a tool which uses SMT predicate abstraction techniques to reduce the number of models, for which an integration tool is called. Symbo \cite{Martires19} uses knowledge compilation to push computational overhead to an offline phase, and subsequently allow efficient online \WMI computation for factorized $w$. Finally, a technique is proposed to compute exact lower and upper bounds for \WMI based on hyper-rectangular decomposition and orthogonal transformations \cite{MerrellAD17}. 

Besides exact solvers, many approximate solvers have been developed for \WMI. For example, a hashing-based approach for \WMI is proposed~\cite{BelleBP15}, which extends existing hashing methods for \WMC, and uses propositional abstraction and requires a polynomial number of NP-oracle calls.
Sampo \cite{Martires19} extends Symbo with Monte Carlo sampling for integral computation, and thus leverages knowledge compilation to quickly evaluate sampled densities. Furthermore, Markov Chain Monte Carlo methods have been applied to \WMI, but such approaches do not provide any guarantees.
The only exception is the general tool of \citeauthor{Chistikov2017}~(\citeyear{Chistikov2017}) for $\#$SMT, which is the unweighted case of \WMI. This tool approximates the solution of a  $\#$SMT  instance by calling a SAT solver. This comes with approximation guarantees as in $\WMC(\CNF)$, using a randomized algorithm that makes polynomially many calls to the SAT solver \cite{Jerrum86}. 

To the best of our knowledge, there is no dedicated study for $\WMI(\DNF)$, despite $\WMC(\DNF)$ being extensively studied, partly motivated from the rich literature on probabilistic data management~\cite{Suciu-PDBs}: query answering in probabilistic databases reduces to $\WMC(\DNF)$ as every conjunctive query is equivalent to a \DNF via its lineage representation. 
The study of $\WMC(\DNF)$ has led to a number of tools, or algorithms from KLM~\cite{Karp89} to hashing-based techniques~\cite{Meel17-DNF}, and recently to neural model counting approaches \cite{ACL-AAAI20}.

Existing tools for $\WMC(\DNF)$ cannot handle extended data models which also include continuous distributions. Such data models are quite common, see e.g.~Monte Carlo Databases (MCDBs) ~\cite{MCDB08},
 and the system PIP~\cite{PIP10}, which extends MCDBs to efficiently query probabilistic data defined over both discrete and continuous distributions.
Abstracting away from subtle differences, these works also introduce approximate inference algorithms, but unlike our approach, these do not provide guarantees. 
The study of $\WMI(\DNF)$ hence serves as a unifying perspective for a class of data models. 

\section{Summary and Outlook}
In this work, we studied weighted model integration on \DNF structures and presented \approxwmi, which is an FPRAS given a concave and factorized weight function $w$. We also presented \approxwmidep, an extension to \approxwmi allowing more relaxed factorizations of $w$. Our FPRAS results for $\WMI(\DNF)$ complement the result of $\WMC(\DNF)$, and help draw a more complete picture of approximability for these problems. Our experimental analysis further shows the potential of \approxwmi.

Looking forward, we aim to investigate alternative approaches for approximate \WMI with guarantees, e.g., hashing-based approaches, to minimize the impact of sampling. We hope this work leads to further investigation of \WMI techniques, leading to more robust \WMI systems. 

\section*{Acknowledgements}

This work was supported by the Alan Turing Institute under the UK EPSRC grant EP/N510129/1, the AXA Research Fund, and by the EPSRC grants EP/R013667/1, EP/L012138/1, and EP/M025268/1. Ralph Abboud is funded by the Oxford-DeepMind Graduate Scholarship and the Alun Hughes Graduate Scholarship. Experiments for this work were conducted on servers provided by the Advanced Research Computing (ARC) cluster administered by the University of Oxford.
\bibliographystyle{named}
\bibliography{library}

\cleardoublepage
\section{Proofs of Main Results}
\label{app:proofs}
\subsection{Proof of Lemma \ref{lem:conf}}
We use the worst-case assumption that a single failure of any oracle, or a failure of the sampling procedure, implies the failure of \approxunion. Hence, we seek to upper-bound the union of these failure probabilities by $\delta$ to ensure \approxunion is an FPRAS. 

In Lemma 3 of \cite{Bringmann}, it is shown that Condition 1 is sufficient for reliable but weak oracles \volumequery, \samplequery, and \pointquery, to ensure \approxunion meets the $\epsilon$ error requirement with probability $\frac{3}{4}$. Hence, we now need to prove that, given unreliable oracles satisfying Condition 2, and failure probability of the LTC sampling procedure generalized from $\frac{1}{4}$ to a value $\delta_1 < \delta$, \approxunion also meets the confidence requirement $\delta$, and is therefore an FPRAS for computing the union of convex bodies.

The generalization of the failure probability of sampling from $\frac{1}{4}$ to $\delta_1 < \delta$ can trivially be achieved by multiplying the required number of trials $T$ specified in Theorem \ref{thm: approxunion} by $\frac{1}{3 \ln(2)}\ln(\frac{2}{\delta_1})$, yielding $T$ as specified in Lemma \ref{lem:conf}. We can now upper-bound the LTC sampling procedure failure probability, denoted $f_{\text{LTC}}$, by setting $\delta_1 = \frac{\delta}{4}$, which yields the value of $T$ shown in Lemma \ref{lem:conf} and Algorithm \ref{alg:wmi}. 
Given $\delta_1 = \frac{\delta}{4}$,  a failure probability of $\frac{3\delta}{4}$ remains, and shall be allocated for all possible oracle failures. We now show that the confidence requirements stated in Condition 2 perform this allocation and indeed upper-bound any oracle failure probability by $\frac{3\delta}{4}$, as required: 

Let $f_v$ denote the failure of any call to \volumequery, $f_s$ the failure of any call to \samplequery, and $f_p$ the failure of any call to \pointquery. Furthermore, let $f$ denote the overall failure probability of \approxunion. Given that \volumequery is called $k$ times within \approxunion, setting $\delta_V \leq \frac{\delta}{4k}$ yields, by the union bound: 
\[\Pr(f_v) \leq k \delta_V = \frac{\delta}{4}.
\]
Within \approxunion, \pointquery is called $T$ times, whereas \samplequery can be called up to $T$ times, depending on the success of \pointquery at the previous iteration. We assume the worst-case and consider that \samplequery is called $T$ times. Applying the union bound with the bound on $\delta_P$ and $\delta_S$ as specified in Condition 2 yields:
\begin{align*}
\begin{split}
&\Pr(f_p \vee f_s) ~\leq~ T (\delta_s) ~\leq~ T \frac{\delta}{2276 \ln(\frac{8}{\delta}) \frac{k}{\epsilon^2}}. 
\end{split}
\end{align*}
We now use the bound from Lemma 3 in \cite{Bringmann}, which gives that, under Condition 1, $T < 2365 \frac{k}{\epsilon^2}$ for failure probability $\frac{1}{4}$. Generalizing this bound for an arbitrary $\delta_1$ gives $T < \frac{2365}{3 \ln(2)} \ln(\frac{2}{\delta_1})\frac{k}{\epsilon^2} < 1138 \ln(\frac{8}{\delta})\frac{k}{\epsilon^2}$. Replacing $T$ with this bound in the failure probability yields 
\[\Pr(f_p \vee f_s) \leq \frac{\delta}{2}.\]

Finally, using the union bound to upper-bound the overall \approxunion failure probability yields
\begin{align*}
\begin{split}
\Pr(f) &\leq Pr(f_v \vee f_p \vee f_s) + \Pr(f_\text{LTC})\\
&\leq \frac{\delta}{4} + \frac{\delta}{2} +\frac{\delta}{4} = \delta,
\end{split}
\end{align*}
as required.
\subsection{Proof of Theorem \ref{thm: wmi_indep_fpras}}
\label{app:proofthm}
The difference between \approxunion and \approxwmi lies in the implementation of task-specific oracles for \WMI. Therefore, it is sufficient to show that all oracles in \approxwmi satisfy the requirements of Lemma \ref{lem:conf} to reach the desired result.

\paragraph{\clauseweight.} We first show the correctness of \clauseweight for computing $\WMI(c, w \mid \Xbf, \Vbf)$ given a concave and factorized $w$. Let $\vbf_c$ be a Boolean assignment satisfying $c$, $\xbf_c$ be the set of all real assignments satisfying $c(\xbf_c, \vbf_{c})$ (i.e., $\xbf_c$ is the polytope induced by the LRA constraints of $c$), and $b_{c}$ be the set of Boolean literals appearing in $c$. Given a factorized $w$, \WMI over $c$ reduces to
\begin{align*}
\begin{split}
\WMI(c, w \mid \Xbf, \Vbf) &= \sum_{\vbf_{c}}~\int\limits_{\xbf_{c}} 
w(\xbf,\vbf_{c})~d \xbf \\
&= \sum_{\vbf_{c}}~\int\limits_{\xbf_{c}} 
w_x(\xbf) \prod_{p \in \vbf_{c}} w_b(p)~d \xbf\\
&= \sum_{\vbf_{c}}  \prod_{p \in \vbf_{c}} w_b(p) ~\int\limits_{\xbf_{c}}  w_x(\xbf)~d \xbf\\  
&=  \Big(\prod_{p \in {b_c}} w_b(p)\Big) ~\int\limits_{\xbf_{c}}  w_x(\xbf)~d \xbf
\end{split}
\end{align*} 
Note that the separation of the sum from the integral in the last step is possible because the set $\xbf_{c}$ is constant across all Boolean assignments, since $c$ is a conjunction of atoms enforcing a unique set of real constraints. 

Hence, the \WMI of $c$ given a factorized, concave $w$ amounts to a product computation of Boolean literal probabilities plus a weighted integral computation over $\xbf_{c}$ in the real domain $\mathbb{R}^n$. In \clauseweight, as shown in Algorithm \ref{alg:functions}, the product computation corresponds to step \clauseweight($\text{c}_1$), whereas weighted integral computation corresponds to steps \clauseweight($\text{c}_2$)  and ($\text{c}_3$).  Weighted integral computation is performed using a transformation of $\xbf_c$ to $\xbf'_c$, as described in Section \ref{ssec:indep}, and it is trivial to prove that the volume of $\xbf'_c$ is equivalent to the weighted integral of $\xbf_{c}$. Therefore, \clauseweight indeed returns an approximation for $\WMI(c, w \mid \Xbf, \Vbf)$. 

We now show that \clauseweight meets the error and confidence criteria stipulated in Lemma \ref{lem:conf}. Multiplication of Boolean weights is error-free, hence it is necessary and sufficient that the call to \volume have error and confidence upper-bounded by $\epsilon_V$ and $\delta_V$ respectively, as is the case in Algorithm \ref{alg:wmi}, for \clauseweight to meet the requirements of Lemma \ref{lem:conf}.

\paragraph{\sample.}  We first show that sampling from the transformation result $\xbf'_c$ is equivalent to sampling from $\xbf_{c}$ according to $w$. Let $x'$ be a real assignment sampled uniformly from $\xbf'_c \in \mathbb{R}^{n+1}$, and let $x$ be $x'$ with the last dimension omitted i.e., $x'_{1:n}$. The weight of $x$ is therefore given by 
	\[
	\int_0^{w_x(x)} \frac{1}{\text{Vol}(\xbf_{B_i})}~dx \sim w_x(x),
	\]
as required. This implies that, when $\xbf'_c$ is successfully sampled with multiplicative error $\epsilon_S$, $x$ is sampled over $\xbf_c$ with weight function $w$ with the same error. 

We now show that \sample meets the requirements of Lemma \ref{lem:conf}. Boolean variable sampling, corresponding to step \sample($\text{s}_1$), occurs with zero error. Hence, it is necessary and sufficient to run \convexbodysampler with parameters $\epsilon_S$ and $\delta_S$ satisfying Condition 2 of Lemma \ref{lem:conf}, to ensure \sample meets the requirements. This is indeed the case with \sample, since \convexbodysampler is called with the same parameters as \sample, and $\epsilon_S$ and $\delta_S$, as specified in Algorithm \ref{alg:wmi}, meet Condition 2. In particular, $\delta_P = 0$, since \evaluate is deterministic, hence  $\delta_S = \frac{\delta}{2276 \ln(\frac{8}{\delta}) \frac{k}{\epsilon^2}}$ is the largest $\delta_S$ satisfying Condition 2.

\paragraph{\evaluate.} This function is deterministic, so trivially meets error and confidence requirements. 

Hence, \approxwmi is an FPRAS for \WMI given a concave and factorized $w$ and given FPRAS functions \clauseweight and \sample, and the deterministic \evaluate. 

\subsection{Proof of Lemma \ref{lem:voldep}}
\label{app:extension}

We first show the correctness of \clauseweightdep for computing $\WMI(c, w~|~\Xbf, \Vbf)$. Under the more general factorization, $w_x$ now depends on both real and Boolean variables. Hence, the decomposition of \WMI computation into two separate Boolean and real parts used in Appendix \ref{app:proofthm} no longer holds. Indeed, \WMI over a conjunction $c$ with $w$ following the more general factorization can only be written as
\begin{align*}
\begin{split}
&\WMI(c, w \mid \Xbf, \Vbf) 
\\&= \sum_{\vbf_{c}}  \prod_{p \in \vbf_{c}} w_b(p) \int\limits_{\xbf_{c}}  w_x(\xbf, \vbf_{c})~d \xbf\\  
&=  \Big(\prod_{p \in {b_c}} w_b(p)\Big) ~\sum_{\vbf_{c}}~\Big(\prod_{p \in \vbf_{c} / b_c} w_b(p)~\int\limits_{\xbf_{c}}  w_x(\xbf, \vbf_{c})~d \xbf\Big).
\end{split}
\end{align*} 
Note that the outer summation no longer simplifies, as the inner integral now depends on the summand with this factorization. Therefore, \WMI computation over $c$ in this setting requires up to $2^{m-|b_{c}|}$ (i.e., the number of Boolean assignments satisfying $c$) calls to a weighted volume computation tool (i.e., transformation to $\xbf'_c \in \mathbb{R}^{n+1}$ then call to \volume). This exponential number of calls is highly prohibitive in practice. Hence, \clauseweightdep uses \emph{Monte-Carlo sampling} to approximate the \WMI of $c$. This sampling is feasible, and can be done with guarantees, since $w$ is $\rho$-restricted, which bounds the set of possible integral values within a small enough range. 

Observe that the final equation for $\WMI(c, w \mid \Xbf, \Vbf)$ can also be written as \[\prod_{p \in {b_c}} w_b(p)~ \mathbb{E}_{\vbf_{c}}\Big[\int\limits_{\xbf_{c}}  w_x(\xbf, \vbf_{c})~d \xbf\Big].\]  Monte-Carlo sampling for the expectation component of $ \WMI(c, w \mid \Xbf, \Vbf)$ is then done as follows:
\begin{enumerate}
	\item  Sample a random Boolean assignment $\tau$ using $w_b$ (Step \clauseweightdep ($\text{c}_1$))
	\item  Compute the resulting weighted model integration over $\xbf_c$ given $w_x(\xbf, \tau)$ using $\volume_{\epsilon,\delta}$. (Steps \clauseweightdep ($\text{c}_{2}$, $\text{c}_{3}$))
\end{enumerate}
Finally, an approximation for $ \WMI(c, w \mid \Xbf, \Vbf)$ is returned by multiplying the average of sample results by the product of Boolean weights, which corresponds to steps \clauseweightdep ($\text{c}_{4}$, $\text{c}_{5}$) in \clauseweightdep. 

Clearly, the results of every sampling step are $s$ independent and identically distributed (i.i.d) random variables. Let $\bar{X}$ denote the mean of these $s$ random variables, and let $\mu =\mathbb{E}_{v_{c}}\Big[\int\limits_{\xbf_c}  w_x(\xbf, \vbf_{c})~d \xbf\Big]$ for ease of notation. From Equation \ref{eq:existence}, we infer that $a(\xbf_c) \leq \mu \leq b(\xbf_c)$. Therefore, we can use $a(\xbf_c)$ and $b(\xbf_c)$ as bounds within the Hoeffding bound for sampling $s$ i.i.d variables. This yields, for a target additive difference $t$, 
\[\Pr\big(|\bar{X} - \mu| \leq t) \leq 2\exp{\Bigg(\frac{-2st^2}{(b(\xbf_c)-a(\xbf_c))^2}\Bigg)}\] 
We now replace $t$ with $\epsilon_\text{Samp} \mu$ to obtain a multiplicative error bound, and upper-bound the failure probability by $\delta_\text{Samp}$ to compute a lower bound for sample complexity:
\begin{align*}
\begin{split}
s &\geq \ln(\frac{2}{\delta_\text{Samp}})\frac{1}{\epsilon_\text{Samp}^2}\frac{(b(\xbf_c) -a(\xbf_c))^2}{\mu^2}\\
&\geq \ln(\frac{2}{\delta_\text{Samp}})\frac{1}{\epsilon_\text{Samp}^2}\frac{(b(\xbf_c) - a(\xbf_c))^2}{a(\xbf_c)^2} \text{~~~as } a(\xbf_c) \leq \mu \leq b(\xbf_c) \\
&\geq \ln(\frac{2}{\delta_\text{Samp}})\frac{1}{\epsilon_\text{Samp}^2}\Bigg(\frac{b(\xbf_c)}{a(\xbf_c)}\Bigg)^2 \text{~~~~~~~~~~~~~~~~~~~~~~~~since } a(\xbf_c) > 0\\
&\geq \ln(\frac{2}{\delta_\text{Samp}})\frac{1}{\epsilon_\text{Samp}^2}\rho^2 \text{~~~~~~~~~~~~~~~~~~~~~~~~~~~~since } \rho = \max_{\xbf_c} \frac{b(\xbf_{c})}{a(\xbf_{c})}
\end{split}
\end{align*}

For \clauseweightdep to be an FPRAS, it must run in polynomial time with respect to $\frac{1}{\epsilon_\text{Samp}}$, $\frac{1}{\delta_\text{Samp}}$, $n$, $m$,  and  $k$. Since a sampling iteration runs in polynomial time (\volume is an FPRAS), then it is necessary and sufficient that $s$ be polynomial in $\frac{1}{\epsilon_\text{Samp}}$, $\frac{1}{\delta_\text{Samp}}$, $n$, $m$,  and  $k$. This is satisfied since $w$ is $\rho$-restricted.

We now show that, under the conditions of Lemma \ref{lem:voldep}, \clauseweightdep produces an $\epsilon, \delta$ approximation of $\WMI(c, w | \Xbf, \Vbf)$. In what follows, we assume that \clauseweightdep fails if Monte-Carlo sampling fails or \emph{any} of the \volume calls fail. Therefore, we first prove that, assuming no failure occurs (i.e. all oracles are reliable), the provided error bounds in Lemma \ref{lem:voldep} produce an estimate within the $\epsilon$ error requirement. Then, we prove that the confidence bounds asserted upper-bound the probability of any \clauseweightdep failure, denoted $f_v$, by $\delta$.
	
\paragraph{Error Bounds.} We first assume a successful run where both Monte-Carlo sampling and \volume produce values within their respective multiplicative error bounds. Setting $\epsilon_\text{Samp}$ and $\epsilon_\text{Comp}$ within the bounds of Lemma \ref{lem:voldep} yields the following bound on $\bar{X}$:
	\[\Big(1 - \frac{\epsilon}{1 + \sqrt{2}}\Big)^2\mu \leq \bar{X} \leq \Big(1 + \frac{\epsilon}{1 + \sqrt{2}}\Big)^2\mu\]
	which, $\forall 0 < \epsilon \leq 1$, implies that
\begin{align*}
\Big(1 - \epsilon\Big)\mu \leq \bar{X} \leq \Big(1 + \epsilon\Big)\mu,
\end{align*}
and, following multiplication by $\Big(\prod_{p \in {b_c}} w_b(p)\Big)$ on all sides, we obtain:
\begin{align*}
&\Big(1 - \epsilon\Big)Q \leq \bar{X} \Big(\prod_{p \in {b_c}} w_b(p)\Big) \leq \Big(1 + \epsilon\Big)Q,
\end{align*}
where $Q=\WMI(c, w \mid \Xbf, \Vbf) $, as required. 
	
\paragraph{Confidence Bounds.} \volume is called $s$ times within \clauseweightdep. Hence, we apply the union bound, to upper-bound $f_v$, and, using the bounds of Lemma \ref{lem:conf},  obtain:
	\begin{align*}
	\begin{split}
	\Pr(f_v) &\leq s\delta_\text{Comp} + \delta_\text{Samp}\\
	&= s\frac{\delta}{2s} + \frac{\delta}{2} = \frac{\delta}{2} = \delta\text{,}
	\end{split}
	\end{align*}
as required.

\end{document}